\documentclass[11pt]{article}
\usepackage{times}
\usepackage{fullpage}
\usepackage{graphicx} 
\usepackage{subfigure} 
\usepackage{amssymb,amsmath,amsthm}
\usepackage{algorithm}
\usepackage{algorithmic}
\usepackage{authblk}
\usepackage{natbib}
\usepackage{hyperref}

\newcommand\sign{\text{sign}}
\newcommand\tp{\text{TP}}
\newcommand\tn{\text{TN}}
\newcommand\fp{\text{FP}}
\newcommand\fn{\text{FN}}
\newcommand\conf{\mathbf{C}}
\newcommand\tph{\widehat{\text{TP}}}
\newcommand\tnh{\widehat{\text{TN}}}
\newcommand\fph{\widehat{\text{FP}}}
\newcommand\fnh{\widehat{\text{FN}}}
\newcommand\tprh{\widehat{\text{TPR}}}
\newcommand\tnrh{\widehat{\text{TNR}}}
\newcommand\prech{\widehat{\text{Prec}}}

\newcommand\confh{\widehat{\mathbf{C}}}
\newcommand{\convinp}{\overset{p}{\to}}
\newcommand\bbP{\mathbb{P}}
\newcommand\cP{\mathcal{P}}
\newcommand\bbR{\mathbb{R}}
\newcommand\s{\mathbf{s}}
\newcommand\y{\mathbf{y}}
\newcommand\x{\mathbf{x}}
\newcommand\z{\mathbf{z}}
\newcommand\X{\mathcal{X}}
\newcommand\Y{\{0, 1\}}
\newcommand\E{\mathbf{E}}
\newcommand\py{p(\y)}
\newcommand\vs{v(\s)}

\newcommand{\MET}{\Psi} 
\newcommand{\ALT}{\Phi}
\newcommand{\REG}{\Gamma}

{}
{}
\newcommand\utilEUM{\mathcal{U}_{\text{EUM}}^{\MET}}
\newcommand\utilDTA{\mathcal{U}^{\MET}}
\newcommand\utilDTAhat{\widehat{\mathcal{U}}^{\MET}}

\def\naive/{na\"{i}ve} 
\def\regular/{TP~monotonic}
\def\Regular/{TP~Monotonic}
\def\regularity/{TP~monotonicity}
\def\Regularity/{TP~Monotonicity}
\def\tpregular/{TPR/TNR~monotonic}
\def\TPRegular/{TPR/TNR~Monotonic}
\def\tpregularity/{TPR/TNR~monotonicity}
\def\TPRegularity/{TPR/TNR~Monotonicity}

\newtheorem{theorem}{Theorem}
\newtheorem*{theorem-non}{Theorem}

\newtheorem{proposition}[theorem]{Proposition}
\newtheorem{corollary}[theorem]{Corollary}

\newtheorem{definition}[theorem]{Definition}
\newtheorem{remark}[theorem]{Remark}

\begin{document}
\title{Optimal Decision-Theoretic Classification Using Non-Decomposable Performance Metrics}
\author[1]{Nagarajan Natarajan} 
\author[2]{Oluwasanmi Koyejo}
\author[1]{Pradeep Ravikumar}
\author[1]{Inderjit S. Dhillon}
\affil[1]{Department of Computer Science, University of Texas, Austin}
\affil[2]{Department of Pyschology, Stanford University}
\renewcommand\Authands{ and }
\date{}
\maketitle
\begin{abstract} 
We provide a general theoretical analysis of expected out-of-sample utility, also referred to as decision-theoretic classification, for non-decomposable binary classification metrics such as F-measure and Jaccard coefficient. Our key result is that the expected out-of-sample utility for many performance metrics is provably optimized by a classifier which is equivalent to a signed thresholding of the conditional probability of the positive class. Our analysis bridges a gap in the literature on binary classification, revealed in light of recent results for non-decomposable metrics in population utility maximization style classification. Our results identify checkable properties of a performance metric which are sufficient to guarantee a probability ranking principle. We propose consistent estimators for optimal expected out-of-sample classification. As a consequence of the probability ranking principle, computational requirements can be reduced from exponential to cubic complexity in the general case, and further reduced to quadratic complexity in special cases. We provide empirical results on simulated and benchmark datasets evaluating the performance of the proposed algorithms for decision-theoretic classification and comparing them to baseline and state-of-the-art methods in population utility maximization for non-decomposable metrics.
\end{abstract} 

\section{Introduction}\label{introduction}

Many binary classification metrics in popular use, such as $F_\beta$ and Jaccard, are \emph{non-decomposable}, which indicates that the utility of a classifier evaluated on  a set of examples cannot be decomposed into the sum of the utilities of the classifier applied to each example. In contrast, decomposable metrics such as accuracy evaluated on set of examples can be decomposed into a sum of per-example accuracies. Non-decomposability of a performance metric is often desirable as it enables a non-linear tradeoff between the overall confusion matrix entries: true positives (TP), false positives (FP), true negatives (TN) and false negatives (FN). As a result, non-decomposable performance metrics remain popular for imbalanced and rare event classification in medical diagnosis, fraud detection, information retrieval applications~\citep{lewis1994, drummond2005, gu2009, he2009}, and in other problems where the practitioner is interested in measuring tradeoffs beyond standard classification accuracy.

A recent flurry of theoretical results and practical algorithms highlights a growing interest in understanding and optimizing non-decomposable metrics~\citep{dembczynski2011exact, ye2012, koyejo2014consistent,narasimhan2014statistical}. Existing theoretical analysis has focused on two distinct approaches for characterizing the \emph{population} version of the non-decomposable metrics: identified by \citet{ye2012} as decision theoretic analysis (DTA) and empirical utility maximization (EUM). DTA population utilities measure the expected gain of a classifier on a fixed-size test set, while EUM population utilities are a function of the population confusion matrix. In other words, DTA population utilities measure the the average utility over an infinite set of test sets, each of a fixed size, while EUM population utilities evaluate the performance of a classifier over a single infinitely large test set. 

It has recently been shown that for EUM based population utilities, the optimal classifier for large classes of non-decomposable binary classification metrics is just the sign of the thresholded conditional probability of the positive class with a metric-dependent threshold \citep{koyejo2014consistent,narasimhan2014statistical}. In addition, practical algorithms have been proposed for such EUM consistent classification based on direct optimization for the threshold on a held-out validation set. In stark contrast to this burgeoning understanding of EUM optimal classification, we are aware of only two metrics for which DTA consistent classifiers have been derived and shown to exhibit a simple form; namely, the $F_\beta$ metric \citep{lewis1995evaluating, dembczynski2011exact, ye2012} and squared error in counting (SEC) studied by \citet{lewis1995evaluating}.

In this paper, we seek to bridge this gap in the binary classification literature, and provide a general theoretical analysis of DTA population utilities for non-decomposable binary classification metrics. Interestingly, we show that for many metrics the DTA optimal classifier again comprises signed thresholding of the conditional probability of the positive class. As we show, for a metric to have such an optimal classifier it must obey the so-called \emph{probability ranking principle} (PRP), which was first formalized by \citet{lewis1995evaluating} in the information retrieval context. We identify a sufficiency condition (a certain monotonicity property) for a metric to obey PRP. We show that these conditions are satisfied by large families of binary performance metrics including the monotonic family studied by \citet{narasimhan2014statistical}, and a large subset of the linear fractional family studied by \citet{koyejo2014consistent}. We also recover known results for the special cases of $F_\beta$ and SEC. 

While the optimal classifiers of both EUM and DTA population utilities associated with the performance metrics we study comprise signed thresholding of the conditional probability of the positive class, the evaluation and optimization for EUM and DTA utilities require quite different techniques. Given a classifier and a distribution, evaluating a population DTA utility can involve exponential-time computation, even leaving aside maximizing the utility on a fixed test set. As we show, in light of the probability ranking principle, and with careful implementation, this can actually be reduced to cubic complexity. These computations can be further reduced to quadratic complexity in a few special cases~\citep{ye2012}. To this end, we propose two algorithms for optimal DTA classification. The first algorithm runs in $O(n^3)$ time for a general metric, where $n$ is the size of the test set and the second algorithm runs in time $O(n^2)$ for special cases such as $F_\beta$ and Jaccard. We show that our overall procedure for decision-theoretic classification is consistent.

\paragraph{Related Work: } 
A full literature survey on binary classification is beyond the scope of this manuscript. We focus instead on some key related results. It is well known that classification accuracy is optimized by thresholding the conditional probability of the positive class at half. \citet{bartlett2006} showed how convex surrogates could be constructed in order to control the probability of misclassification. This work was extended by \citet{steinwart2007} to construct surrogates for asymmetric or weighted binary accuracy. $F_\beta$ is perhaps the most studied of the non-decomposable performance metrics. For instance, \citet{joachims2005support} proposed a support vector machine for directly optimizing the empirical $F_\beta$. \citet{lewis1995evaluating} analyzed the expected $F_\beta$ measure, showing that it satisfied the probability ranking principle. Based on this result, several authors have proposed algorithms for empirical optimization of the expected $F_\beta$ measure including \citet{chai2005expectation}, \citet{jansche2005maximum} and \citet{cheng2010bayes} who studied probabilistic classifier chains.~\citet{ye2012} compared the optimal expected out-of-sample utility and the optimal training population utility for $F_\beta$, showing an asymptotic equivalence as the number of test samples goes to infinity. More recently, \citet{parambath2014optimizing} gave a theoretical analysis of the binary and multi-label $F_\beta$ measure in the EUM setting. ~\citet{dembczynski2011exact} analyzed the $F_\beta$ measure in the DTA setting including the case where the data is non i.i.d., and also proposed efficient algorithms for optimal classification.

\section{Preliminaries}
\label{preliminaries}
Let $X \in \X$ represent instances and $Y \in \Y$ represent labels. We assume that the instances and labels are generated iid as $X,Y \sim \bbP$ for some fixed unknown distribution $\bbP \in \cP$. This paper will focus on non-decomposable performance metrics that are general functions of the entries of the confusion matrix, namely true positives, true negatives, false positives and false negatives. Let bold $\x$ denote a set of $n$ instances $\{x_1, x_2, \dots, x_n\}$ drawn from $\X$, and $\y \in \Y^n$ denote the associated labels. Given a vector of predictions $\s \in \Y^n$ for instances $\x$, the empirical confusion matrix is computed as $\confh(\s,\y) = \begin{bmatrix} \tph & \fnh \\ \fph & \tnh \end{bmatrix}$ with entries:
\begin{align*}
\tph(\s,\y) = \frac{1}{n}\sum_{i=1}^n s_i y_i, \tnh(\s,\y) = \frac{1}{n}\sum_{i=1}^n (1-s_i)(1-y_i)\\
\fph(\s,\y) = \frac{1}{n}\sum_{i=1}^n s_i (1-y_i), \fnh(\s,\y) = \frac{1}{n}\sum_{i=1}^n (1-s_i)y_i.
\end{align*}
To simplify notation, we will omit the arguments when they are clear from context e.g. $\tph$ instead of $\tph(\s,\y)$. 

Let $\MET: [0,1]^4 \mapsto \bbR_+$ denote a non-decomposable metric evaluated on the entries of the confusion matrix. We will sometimes use the abbreviated notation $\MET(\s,\y) := \MET(\confh(\s,\y))$ or $\MET(\confh) := \MET(\confh(\s,\y))$ depending on context. By non-decomposable, we mean that $\MET$ does not decouple as a sum over individual instances $s_i, y_i$. The DTA $\MET$-utility of $\s$ wrt. $\bbP$ is defined as:
\begin{equation}
\utilDTA(\s; \bbP) = \E_{\y \sim \bbP(\cdot|\x)}\MET(\s,\y)
\end{equation}
For the rest of this manuscript, \emph{utility} will refer to the DTA utility unless otherwise noted. 

Note that the development above considered the set of classifier responses $\s \in \{0,1\}^n$ for a given set $\x$ of $n$ input instances. More generally, we are interested in a classifier $\theta: \X \mapsto \Y$, and given a marginal distribution $\bbP_{\X}$ on $\X$, the expected utility of any such classifier $\theta(\cdot)$ can be computed as $\E_{X \sim \bbP_{\X}} \left[ \utilDTA(\s; \bbP(\cdot|\x)) \right]$, where $s_i = \theta(x_i)$. Since the optimal classifier for the expected utility must also optimize $\utilDTA$ pointwise at each $\x$, it is sufficient to analyze the pointwise utility $\utilDTA$ directly. Consequently, we will focus on this quantity for the remainder of the manuscript.

We are thus interested in obtaining the optimal classifier given by:
\begin{equation}\label{eqn:DTA}
\s^* = \underset{\s \in \Y^n}{\arg \max} \ \utilDTA(\s; \bbP). \\
\end{equation}

\begin{remark}[EUM Utility]
\label{rem:EUM}
Fix a classifier $\theta: \X \mapsto \Y$ and a distribution $\bbP \in \cP$, and let $\conf(\theta,\bbP) = \begin{bmatrix} \tp& \fn\\ \fp & \tn \end{bmatrix}$ represent the population confusion matrix with entries:
\begin{align*}
\tp &= \bbP(\theta(x) = 1, y = 1), &\tn &= \bbP(\theta(x) = 0, y = 0), \\
\fp &= \bbP(\theta(x) = 1, y = 0), &\fn &= \bbP(\theta(x) = 0, y = 1).
\end{align*}
EUM utility \citep{koyejo2014consistent,narasimhan2014statistical} is computed as:
$$\utilEUM(\theta; \bbP) = \Psi(\conf(\theta,\bbP))$$ 
i.e. in contrast to the DTA utility, $\MET$ is applied to the population confusion matrix.
\end{remark}


Our analysis will utilize the probability ranking principle (PRP), first formalized by \citet{lewis1995evaluating} as a property of the metric $\MET$ that identifies when the optimal classifier is related to the ordered conditional probabilities of the positive class.
\begin{definition}[Probability Ranking Principle (PRP) \citet{lewis1995evaluating}]
Let $\Psi$ denote a performance metric. We say that $\Psi$ satisfies PRP if for any set $\x$ of $n$ input instances, and any distribution $\bbP(\cdot|\x)$, the optimum  $\s^*$ of the utility \eqref{eqn:DTA} with respect to $\bbP(\cdot|\x)$ satisfies: 
\[ \min\{\bbP(Y = 1 | x_i) | s_i^* = 1\} \geq \max\{\bbP(Y = 1 | x_i) | s_i^* = 0\}.\]
\end{definition}
Let $\sign: \bbR \mapsto \{0,1\}$ as $\sign(t) =1 $ if $t \ge 0$ and $\sign(t) = 0$ otherwise.
The following corollary is immediate.
\begin{corollary}\label{cor:DTAOpt}
Let $\MET$ be a metric for which PRP holds, and let $\x$ denote a set of $n$ iid instances sampled from the marginal $\bbP_{\X}$ of a distribution $\bbP$. The optimal predictions for any such $\x$ is given by the classifier $s_i = \theta^*(x_i) = \sign(\bbP(Y = 1|x_i) - \delta^*)$ where $\delta^* \in [0,1]$ may depend on $\x$.
\end{corollary}
\citet{lewis1995evaluating} showed that PRP holds for a specific non-decomposable measure of practical interest, the $F_\beta$-measure; a similar result was also shown for the squared error in counting (SEC), which is designed to measure the squared difference between the true and the predicted number of positives.

\begin{theorem-non}[\citet{lewis1995evaluating}] 
\begin{enumerate}
  \item PRP holds for $F_\beta$ defined as:
\begin{equation}\label{eq:Fbeta}
\MET_{F_\beta}(\confh) = \frac{(1+\beta^2)\tph}{(1+\beta^2)\fnh +\beta^2\fnh + \fph}.
\end{equation}
\item PRP holds for SEC defined as:
\begin{equation*}\label{eq:sec}
\MET_{\text{SEC}}(\confh) = \left( p - v \right)^2 = (\fnh - \fph)^2.
\end{equation*}
where $p := \frac{1}{n}\sum_i y_i= \tph + \fnh$ and $v := \frac{1}{n}\sum_i s_i = \tph + \fph$.
\end{enumerate}
\end{theorem-non}

\section{PRP for General Performance Metrics}\label{sec:generalizedPRP}
PRP is a meaningful property for any performance metric since, as a consequence of Corollary~\ref{cor:DTAOpt}, any metric satisfying PRP admits an optimal classifier with a simple form. In this section, we identify sufficient conditions for a metric $\MET$ to satisfy PRP. To begin, we consider the following equivalent representation for any metric $\MET$.
\begin{proposition}
\label{prop:tilpsi}
Let $u = \tph(\s,\y), v = v(\s) := \frac{1}{n}\sum_i s_i$ and $p = p(\y) := \frac{1}{n}\sum_i y_i$, then $\exists\; \ALT: [0,1]^3 \to \bbR_+$ such that:
\begin{equation}
\MET(\confh(\s, \y)) = \ALT(\tph(\s,\y), v(\s), p(\y)).
\end{equation}
\end{proposition}
Next, we consider a certain monotonicity property which we have observed is satisfied by popular binary classification metrics.
\begin{definition}[\Regularity/] A metric $\MET$ is said to be \regular/ if when $u_1 > u_{2}$ and $v, p$ fixed, it follows that $\ALT(u_{1,}v,p) > \ALT(u_{2},v,p)$.
\end{definition}
In other words, $\MET$ satisfies \regularity/ if the corresponding representation $\ALT$ (Proposition~\ref{prop:tilpsi}) is monotonically increasing in its first argument.

For any $\MET$, \regularity/ may be verified by applying the representation of Proposition~\ref{prop:tilpsi}. It is easy to verify, for instance that $\ALT_{F_\beta}(u,v,p) = \frac{(1+\beta^2)u}{\beta^2p + v}$ is monotonic in $u$. Our analysis will show that the \regularity/ property is sufficient to guarantee that $\MET$ satisfies PRP. The proof is provided in Appendix~\ref{app:mainresult1}.
\begin{theorem}[Main Result 1] 
\label{thm:mainresult1}
The probability ranking principle holds for any $\MET$ that satisfies \regularity/.
\end{theorem}

While \regularity/ of $\MET$ is sufficient for PRP to hold, it is not necessary. For instance, consider the subclass of performance metrics where $\ALT(\cdot,v,p)$ is independent of the first argument i.e. independent of $\tph$. SEC is an example of a performance metric in this family with $\ALT_{\text{SEC}}(\tph, v, p) = v + p$. The following proposition shows that such metrics also satisfy PRP.
\begin{proposition}\label{prop:sec}
Let $\MET = \ALT(\tph,v,p)$ be a performance metric independent of $\tph$, then $\MET$ satisfies PRP.
\end{proposition}
\begin{proof}
Suppose $\ALT(\cdot,v,p)$ is independent of its first argument. Let $\s^*$ be an optimal classifier, with $v* = v(\s^*)$. If $\s^*$ does not satisfy PRP, then sort $s^*$ with respect to $\bbP(Y|x_i)$ to obtain a new classifier $\tilde{\s}$. It is clear that $v(\s^*) = v(\tilde{\s})$, and $\ALT(\cdot,v(\s^*),p) = \ALT(\cdot,v(\tilde{\s}),p)$, so $\tilde{\s}$ is also an optimal classifier which satisfies PRP.
\end{proof}


\subsection{Recovered and New Results}\label{sec:examples}
This section outlines a few examples of known and new results recovered via the application of Theorem~\ref{thm:mainresult1}, which include a subset of the fractional linear family of \citet{koyejo2014consistent} and the family of performance metrics studied by \citet{narasimhan2014statistical}.

\paragraph{The Fractional Linear Family:} \citet{koyejo2014consistent} studied a large family of performance metrics, and showed that their EUM optimal classifiers are given by the thresholded sign of the marginal probability of the positive class. This family contains, for example the $F_\beta$ and Jaccard measures. The family $\MET_{\text{FL}}$ is equivalently represented by:
\begin{equation}\label{eq:ratio}
\ALT_{\text{FL}}(\tph(\s, \y), v(\s), p(\y)) = 
\frac{c_0+ c_1\tph + c_2 v + c_3 p}{d_0 + d_1\tph + d_2 v + d_3 p}
\end{equation}
for bounded constants $c_i, d_i$, $i = \{0,1,2,3\}$. Our analysis identifies a subclass of this family that satisfies PRP. The following result can be proven by inspection and is stated without proof.
\begin{proposition}\label{prop:lf}
If $c_1>d_1$, then $\MET_{\text{FL}}$ satisfies \regularity/. 
\end{proposition}

\paragraph{Performance Metrics from \citet{narasimhan2014statistical}:} 

An alternative three-parameter representation of metrics $\MET$ was studied by \citet{narasimhan2014statistical} as described in the following proposition.

\begin{proposition}[\citet{narasimhan2014statistical}]
\label{prop:tiltilpsi}
Let $p = p(\y) := \frac{1}{n}\sum_i y_i$, $r_p = \tprh(\s, \y) = \frac{\tph(\s,\y)}{p(\y)}$ and $r_n = \tnrh(\s, \y) = \frac{\tnh(\s,\y)}{1-p(\y)}$, then $\exists\; \REG: [0,1]^3 \to \bbR_+$ such that:
\begin{equation}
\MET(\confh(\s, \y)) = \REG(\tprh(\s,\y), \tnrh(\s,\y), p(\y)).
\end{equation}
\end{proposition}
As shown in Table~\ref{tab:measures}, many performance metrics used in practice are easily represented in this form.
Representation for additional metrics is simplified by including the empirical precision, given by $\prech(\s, \y) = \frac{\tph(\s,\y)}{\vs}$, where $v(\s) := \frac{1}{n}\sum_i s_i = \tph + \fph$ can be computed from the quantities in Proposition~\ref{prop:tiltilpsi}.

Consider the following monotonicity property relevant to the representation in  Proposition~\ref{prop:tiltilpsi}.
\begin{definition}[\TPRegularity/] A metric $\MET$ is said to be \tpregular/ if when $r_{p1}>r_{p2}$ and $r_{n1}>r_{n2}$ and $p$ fixed, it follows that $\REG(r_{p1},r_{n1},p) > \REG(r_{p2},r_{n2},p)$.
\end{definition}
In other words, $\MET$ satisfies \tpregularity/ if the corresponding representation $\REG$ (Proposition~\ref{prop:tiltilpsi}) is monotonically increasing in its first two arguments. It can be shown that all the measures listed in Table \ref{tab:measures} satisfy \tpregularity/. Further, \citet{narasimhan2014statistical} showed that given additional smoothness conditions on $\bbP$, the associated metrics $\REG$ admit an optimal EUM classifier with the familiar signed thresholded form. 

The following proposition shows that any performance metric that satisfies \tpregularity/ also satisfies \regularity/. Thus, \regularity/ is a weaker condition. The proof is provided in Appendix~\ref{app:connection}.
\begin{proposition}
\label{prop:connection}
If $\MET$ satisfies \tpregularity/, then $\MET$ satisfies \regularity/. 
\end{proposition}

It follows from Corollary~\ref{cor:DTAOpt} that any metric that satisfies \tpregularity/ admits a DTA optimal classifier that takes the familiar signed-threshold form. We can verify from the third column of Table~\ref{tab:measures} that each of the \tpregular/ measures $\ALT(u, v, p)$ is monotonically increasing in $u$. 
\begin{remark}
\Regularity/ is a strictly weaker condition than \tpregularity/. Consider the following counterexample, where $\MET(\s,\y) = 2\tph(\s,\y) + \fph(\s,\y)$ with equivalent representation given by $\ALT(\s,\y) = \tph(\s,\y) + v(\s)$ and $\REG(\s,\y) = 2p(\y)\tprh(\s,\y) - (1-p(\y))\tnrh(\s,\y) - p + 1$. Clearly $\MET$ is \regular/, but not \tpregular/.
\end{remark}

\begin{table*}[t]
\caption{Performance metrics for which probability ranking principle (PRP) holds. The third column expresses each measure $\MET(\s, \y)$ as $\ALT(\tph, \vs, \py)$.}
\label{tab:measures}
\begin{center}
{\renewcommand{\arraystretch}{1.4}
\begin{tabular}{|l|l|l|} \hline
\textbf{\textsc{Metric}} & \textbf{\textsc{Definition}} & $\ALT(u,v,p)$ \\ \hline
AM & $(\tprh + \tnrh)/2$ & $\frac{u + p(1 - v - p)}{p(1-p)}$ \\ 
$F_\beta$ & $\big(1+\beta^2\big)/\big(\frac{\beta^2}{\prech} + \frac{1}{\tprh}\big)$ & $\frac{(1+\beta^2)u}{\beta^2 p + v}$\\
\text{Jaccard} & $\tph/\big(\tph + \fph + \fnh\big)$ & $\frac{u}{p + v - u}$ \\
G-TP/PR & $\sqrt{\tprh . \prech}$ & $\frac{u}{\sqrt{p . v}}$ \\
G-Mean & $\sqrt{\tprh . \tnrh}$ & $\frac{u(1-v-p+u)}{p(1-p)}$ \\
H-Mean & $2/\big(\frac{1}{\tprh} + \frac{1}{\tnrh} \big)$ & $\frac{2u(1-v-p+u)}{(1-v-p)p + u}$ \\
Q-Mean & $1 - \frac{1}{2}\big((1 - \tprh)^2 + (1 - \tnrh)^2\big)$ & $1 - \frac{1}{2}\big((\frac{p - u}{2})^2 + (\frac{v - u}{2})^2\big)$ \\ \hline
\end{tabular}
}
\end{center}
\end{table*}

\section{Algorithms}
\label{sec:algos}
In this section, we present efficient algorithms for computing DTA optimal predictions for a given set of instances $\x$ and a non-decomposable performance measure $\Psi$ that satisfies PRP. We also examine the consistency of the proposed algorithms. Apriori, solving \eqref{eqn:DTA} is NP-hard. The key consequence of Theorem \ref{thm:mainresult1} is that we do \emph{not} have to search over $2^n$ possible label vectors to compute the optimal predictions. In light of Corollary~\ref{cor:DTAOpt}, it suffices to consider $n+1$ prediction vectors that correspond to selecting top $k$ instances as positive, after sorting them by $\bbP(Y=1|x)$, for some $k$. Even when $\bbP(Y=1 | x)$ is known exactly, it is not obvious how to compute the expectation in \eqref{eqn:DTA} without exhaustively enumerating $\y$ vectors. We now turn to address these computational questions.

\subsection{$O(n^3)$ Algorithm for PRP Measures}
\citet{ye2012} suggest a simple trick to compute the expectation in $O(n^3)$ time for the $F_\beta$-measure. We make the observation that by evaluating $\MET$ through $\ALT$, we can essentially use the same trick to obtain a cubic-time algorithm to solve \eqref{eqn:DTA} for general measures $\MET$ satisfying the probability ranking principle. Consider the vector $\s \in \{0,1\}^n$ with the top $k$ values set to 1 and the rest to 0, and let $S_{i:j} := \sum_{l=i}^j y_l$. Note that any $\y\in \{0,1\}^n$ that satisfies $S_{1:k} = k_1$ and $S_{k+1:n} = k_2$, $\Psi(\s, \y)$ can simply be evaluated as $\ALT(\frac{1}{n}k_1, \frac{1}{n}k, \frac{1}{n}(k_1 + k_2))$. Thus $\utilDTA(\s,\bbP) = \sum_{\y \in \{0,1\}^n} \bbP(\y|\x)\MET(\s, \y)$ can be evaluated as a sum over possible values of $k_1$ and $k_2$, where the expectation is computed wrt. $P(S_{1:k} = k_1)P(S_{k+1:n} = k_2)$ with $0 \leq k_1 \leq k$ and $0 \leq k_2 \leq n-k$. Now, it remains to compute $P(S_{1:k} = k_1)$ and $P(S_{k+1:n} = k_2)$ efficiently. 

Let $\eta_i = \bbP(Y_i=1 | x_i)$. A consistent estimate of this quantity may be obtained by minimizing a strongly proper loss function such as logistic loss~\cite{reid2009surrogate}. Using the iid assumption on the draw of labels, we can show that $P(S_{1:k} = k_1)$ and $P(S_{k+1:n} = k_2)$ are the coefficients of $z^i$ in $\Pi_{j=1}^k [\eta_j z + (1-\eta_j)]$ and $\Pi_{j=k+1}^n [\eta_j z + (1-\eta_j)]$, each of which can be computed in time $O(n^2)$ for fixed $k$. Note that the metric $\MET$ can be evaluated in constant time. The resulting $O(n^3)$ algorithm is presented in Algorithm \ref{alg:DTAgeneral}. The overall method is as follows:
\begin{enumerate}
\item First, obtain an estimate of $\eta_i = \bbP(Y_i=1 | x_i)$ e.g. via logistic regression.
\item Re-order indices in the descending order of estimated $\eta_i$'s.
\item Then, invoke Algorithm \ref{alg:DTAgeneral} with the sorted $\eta_i$'s to compute $\s^*$.
\end{enumerate}

\begin{algorithm}[tb]
   \caption{Computing $\s^*$ for PRP $\MET$}
\label{alg:DTAgeneral}
\begin{algorithmic}[1]
 \STATE {\bfseries Input:} $\Psi$ and estimates of $\eta_i$ for instances $x_i$ with indices $i = 1,2,\dots,n$ sorted wrt. $\eta_i$
 \STATE Init $s^*_i = 0, \forall i \in [n]$.
 \FOR{$k = 1$ to $n$}
 \STATE For $0 \leq i \leq k$, set $C_k[i]$ as the coefficient of $z^i$ in $\Pi_{i=1}^k\big(\eta_iz + (1-\eta_i)\big)$.
 \STATE For $0 \leq i \leq n-k$, set $D_k[i]$ as the coefficient of $z^i$ in $\Pi_{i=k+1}^n\big(\eta_iz + (1-\eta_i)\big)$.
 \STATE $\Psi_k \leftarrow \hspace{-.3cm}\underset{0 \leq k_2 \leq n-k}{\displaystyle\sum_{0 \leq k_1 \leq k}}C_k[k_1]D_k[k_2]\ALT(\frac{1}{n}k_1, \frac{1}{n}k, \frac{1}{n}(k_1+k_2))$.
 \ENDFOR
 \STATE Set $k^* \leftarrow \arg\max_k \Psi_k$ and $s^*_i \leftarrow 1$ for $i \in [k^*]$.
\STATE return $\s^*$
\end{algorithmic}
\end{algorithm}

\subsection{$O(n^2)$ Algorithm for a Subset of Fractional-Linear Metrics}
We focus our attention on the fractional-linear family of non-decomposable performance metrics studied by \citet{koyejo2014consistent}. Recall that a fractional-linear metric can be represented by $\ALT_{\text{FL}}$ as given in \eqref{eq:ratio}. As shown in in Proposition~\ref{prop:lf}, $\MET_{\text{FL}}$ satisfies \regularity/ when $c_1 > d_1$.
For certain measures in the $\MET_{\text{FL}}$ family, we can get a more efficient algorithm for solving \eqref{eqn:DTA}. In particular, when $c_3 = 0$ in \eqref{eq:ratio}, we can give a quadratic-time procedure for computing $\s^*$ that generalizes the method proposed by~\citet{ye2012} when the constants $\{d_0, d_1, d_2, d_3\}$ are rational. Formally, we consider the sub-family of \regular/ fractional-linear metrics:
\begin{align}
\label{eq:FLspecial}
\{\Psi_{SFL}: \ALT_{FL}(u,v,p) = \frac{c_0 + c_1 u + c_2 v}{d_0 + d_1 u + d_2 v + d_3 p},  \ c_1 > d_1, \text{ and } d_0, d_1, d_2, d_3 \text{ are rational}\}.   
\end{align}
Consider Step 6 of Algorithm \ref{alg:DTAgeneral} for a measure in family \eqref{eq:FLspecial}: \[ \Psi_k \leftarrow \sum_{0 \leq k_1 \leq k}C[k_1](c_0n + c_1k_1 + c_2k)\sum_{0 \leq k_2 \leq n-k} D[k_2]/(d_0n + (d_1+d_3)k_1 + d_2k + d_3k_2).\] Define $b(k, \alpha) = \sum_{0 \leq k_2 \leq n-k} D_k[k_2]/(\alpha + d_3k_2)$. Verify that $b(n, \alpha) = 1/\alpha$. From the fact that $D_{k-1}[i] = \eta_kD_k[i-1] + (1-\eta_k)D_k[i]$, it follows that:
\[ b(k-1,\alpha) = \eta_k b(k, \alpha + d_3) + (1 - p_k)b(k, \alpha).\]
Now, when $d_i$'s are rational, i.e. $d_i = q_i/r_i$,  the above induction can be implemented using an array to store the values of $b$, for possible values of $\alpha$. The resulting $O(n^2)$ algorithm is presented in Algorithm \ref{alg:DTA_FL}. Algorithm \ref{alg:DTA_FL} applies to the $F_\beta$ as well as the Jaccard measure listed in Table \ref{tab:measures}.

\paragraph{Correctness of Algorithm~\ref{alg:DTA_FL}:} When $d_3 \neq 0$, at line 7 of Algorithm~\ref{alg:DTA_FL}, we can verify that  $S[i] = b(k, (i+j_0n)d_3/j_{u,2})$, and therefore at line 9, $S[(j_{u,1} + j_{u,2})k_1 + j_v k] = b(k, (j_{u,1} + j_{u,2})k_1 + j_v k+j_0n)d_3/j_{u,2}) = b(k, (d_1+d_3)k_1 + d_2k + d_0n)$ as desired. When $d_3 = 0$, $b(k, \alpha) = b(k-1, \alpha)$ for all $1 \leq k \leq n$. Let $q_3 = 0$ and $r_3 = 1$. Then, line 5 sets $S[i] = r_0r_1r_2/(i+j_0n)$, line 11 maintains this invariant as $j_{u,2} = 0$ in this case, and therefore at line 9, $S[(j_{u,1} + j_{u,2})k_1 + j_v k] = 1/(d_1k_1 + d_2k + d_0n)$ as desired.

\begin{algorithm}[tb]
   \caption{Computing $\s^*$ for $\MET_{\text{SFL}}$ in the family \eqref{eq:FLspecial}}
\label{alg:DTA_FL}
\begin{algorithmic}[1]
 \STATE {\bfseries Input:} Estimates $\eta_i$ for instances $x_i$, $i = 1,2,\dots,n$ sorted wrt. $\eta_i$, and $c_0, c_1, c_2, d_i = q_i/r_i, i = 0,1,2,3$ corresponding to $\ALT_{\text{SFL}}$
 \STATE Init $s^*_i = 0, \forall i \in [n]$.
 \STATE Set $j_0 \leftarrow r_1r_2r_3q_0,\ j_{u,1} \leftarrow r_0r_2r_3q_1,\  j_{u,2} \leftarrow r_0r_1r_2q_3$,\  $j_v \leftarrow r_0r_1r_3q_2$
 \FOR{$1 \leq i \leq (|j_{u,1}| + |j_{u,2}| + |j_v|)n$}
 \STATE set $S[i] \leftarrow r_0r_1r_2r_3/(i+j_0 n)$.
 \ENDFOR
 \FOR{$k = n$ to $1$}
 \STATE For $0 \leq i \leq k$, set $C_k[i]$ as the coefficient of $z^i$ in $\Pi_{i=1}^k\big(\eta_iz + (1-\eta_i)\big)$.
 \STATE $\Psi_{\text{SFL};k} \leftarrow \hspace{-.3cm}{\displaystyle\sum_{0 \leq k_1 \leq k}}(c_0n + c_1k_1 + c_2k) C_k[k_1]S[(j_{u,1}+j_{u,2})k_1 + j_vk]$.
 \FOR{$i = 1$ to $(|j_{u,1}| + |j_{u,2}| + |j_v|)(k-1)$}
 \STATE $S[i] \leftarrow (1-\eta_k)S[i] + \eta_kS[i+j_{u,2}]$.
 \ENDFOR
 \ENDFOR
 \STATE Set $k^* \leftarrow \arg\max_k \MET_{\text{SFL};k}$ and $s^*_i \leftarrow 1$ for $i \in [k^*]$.
\STATE return $\s^*$
\end{algorithmic}
\end{algorithm}

%

\paragraph{Consistency:} Consider a procedure that maximizes the utility $\utilDTA(\s, \hat{\bbP}(Y|\x))$ computed with respect to a consistent estimate $\hat{\bbP}(Y|\x)$ of the probability $\bbP(Y|\x)$. Here, we show that any such procedure is consistent. The proof is provided in Appendix~\ref{app:consistent}.
\begin{theorem}\label{thm:consistent}
Let $\eta(x_i) = \bbP(Y=1|x_i)$, and assume the estimate $\hat{\eta}(x_i)$ satisfies $\eta(x_i) \convinp \eta(x_i)$. Given a bounded performance metric $\MET$ and a fixed test set of size $n$, let $\s^* = {\arg \max}_{\s \in \{0,1\}^n} \; \utilDTA(\s; \bbP(Y|\x)) $ be the utility optimal prediction with respect to $\bbP$ and $\hat{\s} = {\arg \max}_{\s \in \{0,1\}^n} \; \utilDTA(\s; \hat{\bbP}(Y|\x))$ be the utility optimal prediction with respect to the consistent estimate $\hat{\bbP}(Y|\x)$, then
\[ \utilDTA(\s^*; \bbP) - \utilDTA(\hat{\s}; \bbP) \convinp 0. \]
\end{theorem}

As stated in Theorem~\ref{thm:consistent}, consistency of DTA utility maximization with empirical probability estimates does not depend on PRP. Thus, the consistency results also apply to previous algorithms proposed for $F_{\beta}$ e.g. by \citet{lewis1995evaluating,chai2005expectation,jansche2005maximum,ye2012} that did not include an analysis of consistency with empirical probability estimates. In the special case of \regular/ performance metrics, the following corollary, which follows directly from Theorem~\ref{thm:consistent}, shows that Algorithm \ref{alg:DTAgeneral} and Algorithm \ref{alg:DTA_FL} are consistent.
\begin{corollary}
Assume the estimate $\hat{\eta}(x)$ satisfies $\eta(x) \convinp \eta(x)$ and the performance metric $\MET$ that is \regular/. For a fixed test set of size $n$, let $\hat{\s}$ denote the output of Algorithm \ref{alg:DTAgeneral} (or Algorithm \ref{alg:DTA_FL}, where applicable) using the empirical estimate $\hat{\eta}(x_i)$. Then
\[ \utilDTA(\s^*; \bbP) - \utilDTA(\hat{\s}; \bbP) \convinp 0, \]
where $\s^*$ is the optimal prediction computed with respect to the true distribution $\eta(x_i) = \bbP(Y = 1|x_i)$.
\end{corollary}

\section{Experiments}
\label{experiments}
We present two sets of experiments. The first is an experimental validation on synthetic data with known ground truth probabilities. The results serve to verify the probability ranking principle (Theorem~\ref{thm:mainresult1}) for some of the metrics in Table \ref{tab:measures}. The second set is an experimental evaluation of DTA optimal classifiers on benchmark datasets, and includes a comparison to EUM optimal classifiers and standard empirical risk minimization with a fixed threshold of $1/2$ -- designed to optimize classification accuracy. 

\subsection{Synthetic data: PRP for general metrics}
We consider four metrics from Table \ref{tab:measures} namely AM, Jaccard, $F_1$ (harmonic mean of Precision and Recall) and G-TP/PR (geometric mean of Precision and Recall) which satisfy PRP from Theorem \ref{thm:mainresult1}. To simulate, we sample a set of ten $2$-dimensional vectors $\x = \{x_1, x_2, \dots, x_{10}\}$ from the standard Gaussian. The conditional probability is modeled using a sigmoid function: $\eta_i = \bbP(Y = 1|x_i) = \frac{1}{1+\exp{-w^Tx_i}}$, for a random vector $w$ also sampled from the standard Gaussian. The optimal predictions $\s^*$ that maximize the DTA objective \eqref{eqn:DTA} are then obtained by exhaustive search over the $2^{10}$ possible label vectors. For each metric, we plot the conditional probabilities (in decreasing order) and $\s^*$ in Figure \ref{fig:synth}. We observe that PRP holds in each case (Algorithms \ref{alg:DTAgeneral} and \ref{alg:DTA_FL} produce identical results; plots not shown).

\begin{figure*}
        \centering
        \subfigure[AM]{\includegraphics[width=0.25\textwidth]{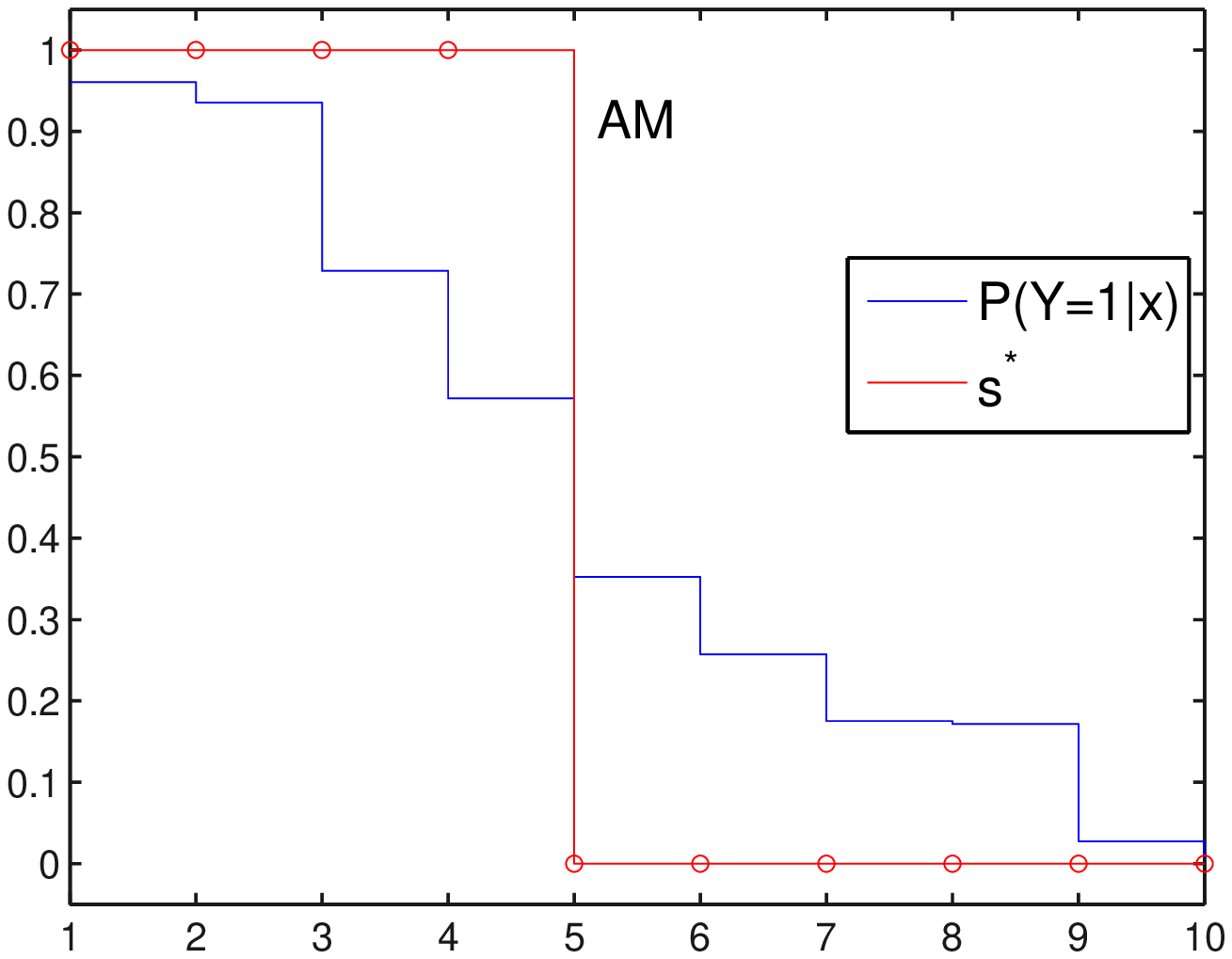}}\hspace{-0.1cm}
        \subfigure[$F_{1}$]{\includegraphics[width=0.25\textwidth]{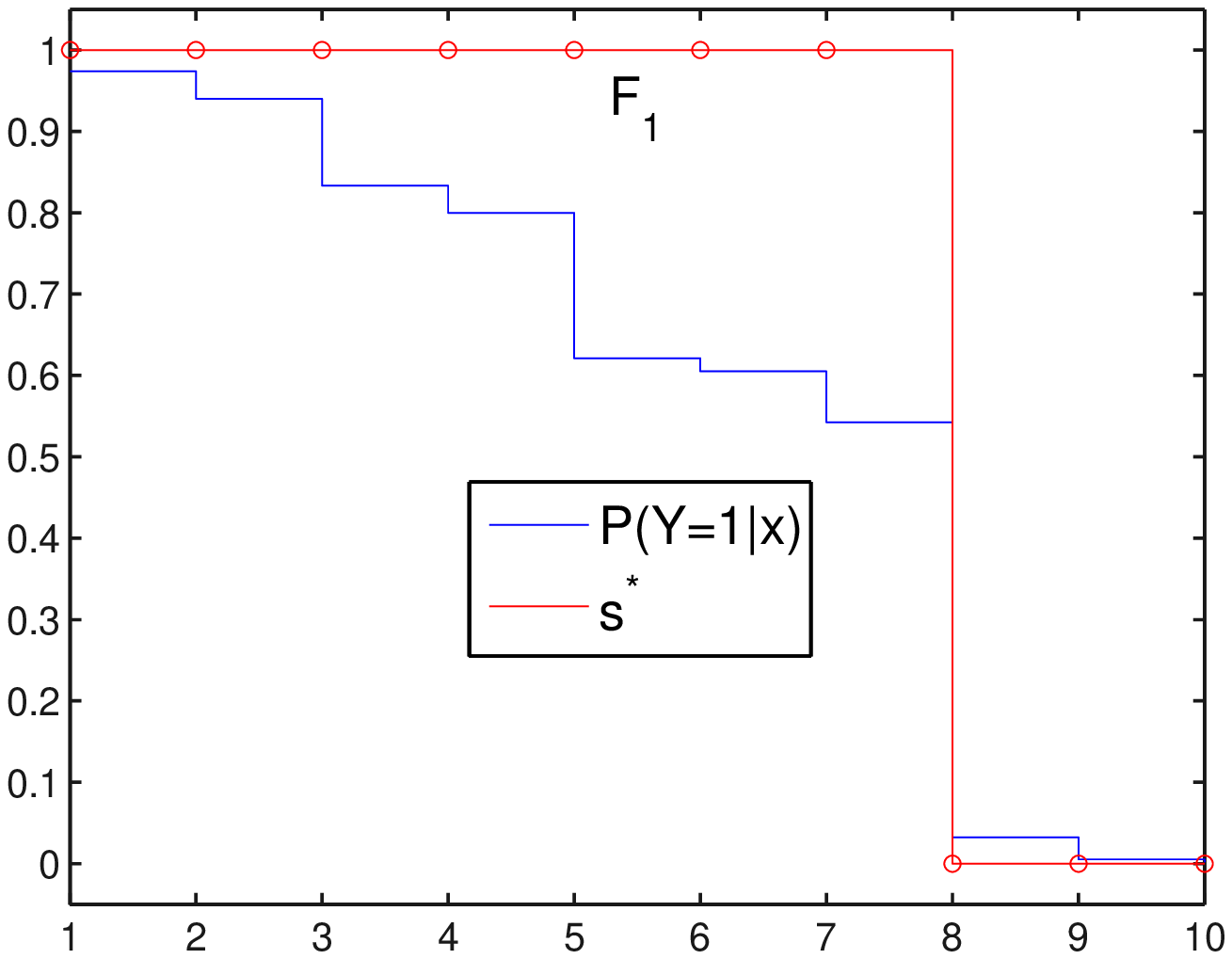}}\hspace{-0.1cm}
        \subfigure[Jaccard]{\includegraphics[width=0.25\textwidth]{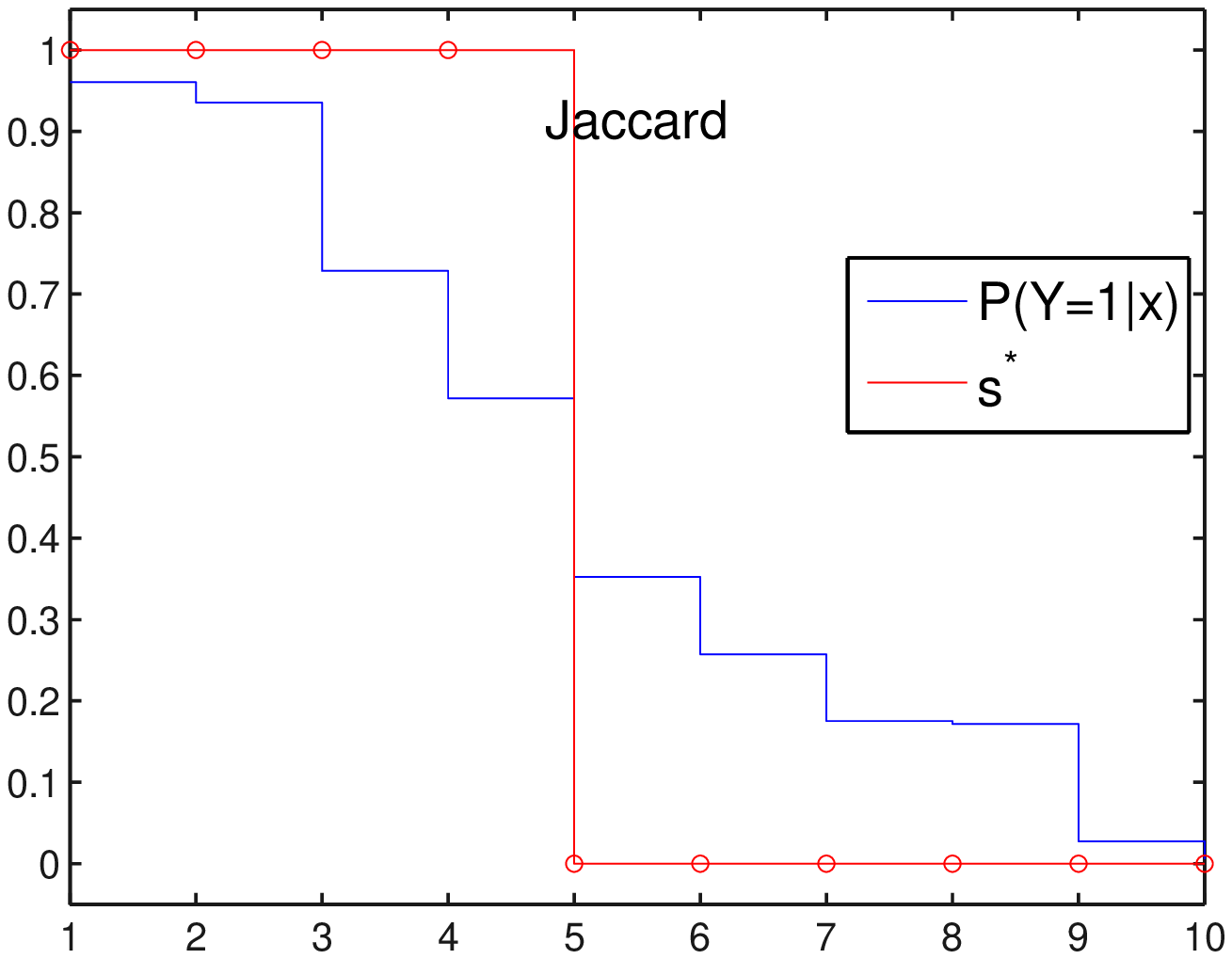}}\hspace{-0.1cm}
        \subfigure[G-TP/PR]{\includegraphics[width=0.25\textwidth]{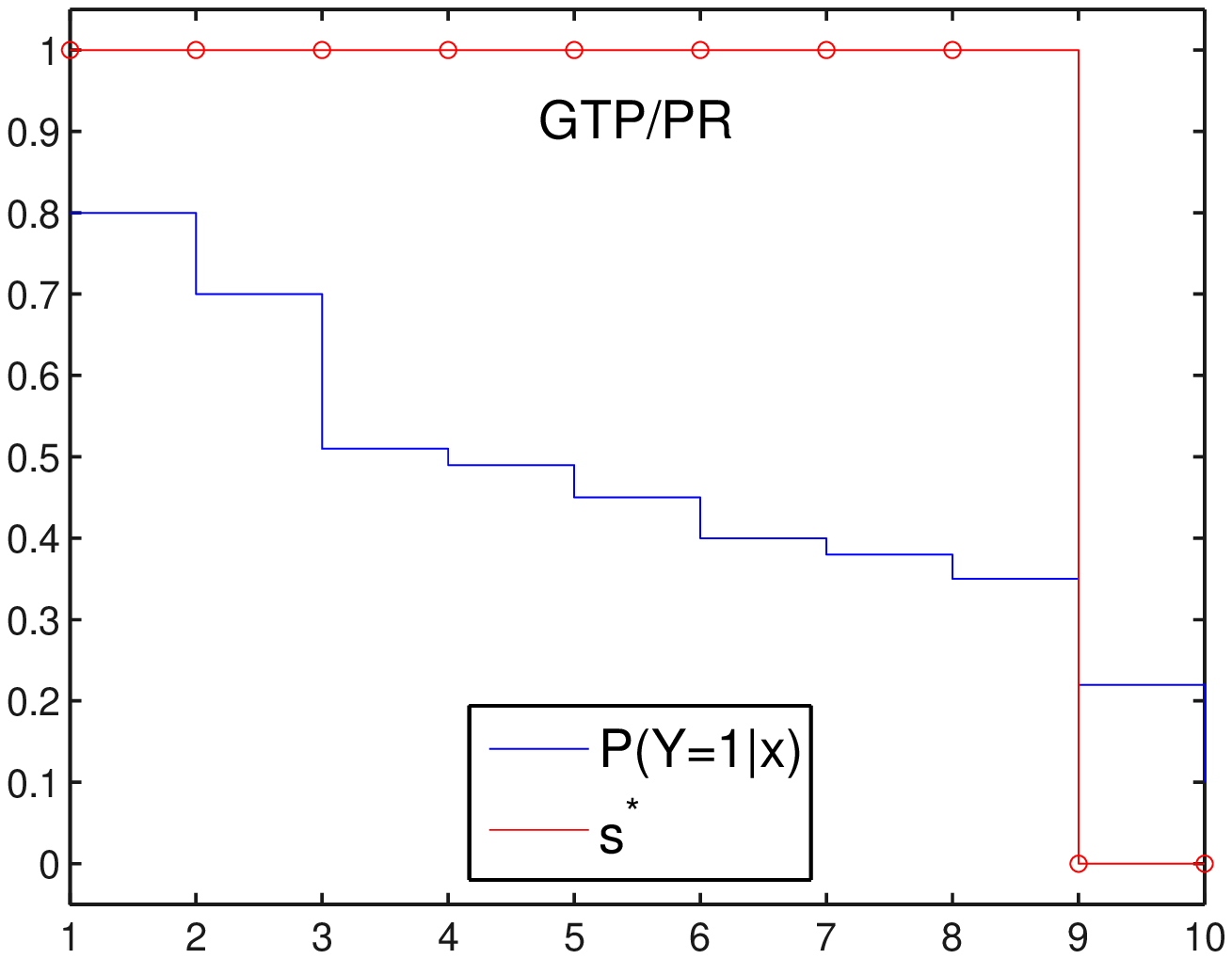}}
\caption{PRP of metrics from Table \ref{tab:measures} demonstrated on synthetic data. In each case, we verify that $\s_i^*$ is obtained by thresholding $\eta(x_i)$ at a fixed value. Furthermore, different measures are optimized at different thresholds on $\x$ from the same distribution $\bbP$.}
\label{fig:synth}
\end{figure*}

\subsection{Benchmark data: Evaluation of the proposed algorithms}
We perform DTA classification using the proposed approach (i) obtain a model for the conditional distribution $\eta(x) = \bbP(Y = 1|x)$ using training data and (ii) compute compute $\s^*$ for the test data using estimated conditionals in the proposed Algorithms \ref{alg:DTAgeneral} and \ref{alg:DTA_FL}. We use logistic loss on the training samples (with $L_2$ regularization) to obtain an estimate $\hat{\eta}(x)$ of $\bbP(Y = 1|x)$. In our experiments, we consider the four performance metrics AM, $F_1$, Jaccard and G-TP/PR. For AM and G-TP/PR we use Algorithm \ref{alg:DTAgeneral}, while for the fractional-linear metrics Jaccard and $F_1$ we use the more efficient Algorithm \ref{alg:DTA_FL}. Let $\y^*$ denote the true labels for the test data. We report the achieved held-out utility $\MET(\confh(\s^*,\y^*))$.

We compare DTA classification using the aforementioned metrics with that of the EUM classifiers using the corresponding metrics as discussed in Remark \ref{rem:EUM}. We use the plugin-estimator method proposed by \citet{koyejo2014consistent} and \citet{narasimhan2014statistical}, where the optimal classifier is given by $\sign(\hat{\eta}(x) - \delta)$. The training data is split into two sets, one set is used for estimating $\hat{\eta}(x)$ and the other for selecting the optimal $\delta$. The predictions are then made by thresholding $\hat{\eta}(x)$ of the test data points at $\delta$. We also compare to the baseline method of thresholding $\hat{\eta}(x)$ at 1/2.

We report results on seven benchmark datasets (used in \cite{koyejo2014consistent,ye2012}). (1) \textsc{Reuters}, consisting of 8293 news articles categorized into 65 topics. Following~\citep{ye2012,koyejo2014consistent}, we present results for averaging over topics that had at least $T$ positives in the training (5946 articles) as well as the test (2347 articles) data; (2) \textsc{Letters} dataset consisting of 20000 handwritten characters (16000 training and 4000 test instances) categorized into 26 letters; (3) \textsc{Scene} (a UCI benchmark dataset) consisting of 2230 images (1137 training and 1093 test instances) categorized into 6 scene types; (4) \textsc{Webpage} binary dataset, consisting of 34780 web pages (6956 train and 27824 test); highly imbalanced, with only about 182 positive instances in the train; (5) \textsc{Image}, with 1300 train and 1010 test images; (6) \textsc{Breast Cancer}, with 463 train and 220 test instances, and (7) \textsc{Spambase} with 3071 train and 1530 test instances\footnote{See~\cite{koyejo2014consistent,ye2012} for more details on the datasets}. The results for $F_1$ and Jaccard metrics (using Algorithm \ref{alg:DTA_FL} for DTA) are presented in Table \ref{fig:table1}. We find that DTA classifier which optimizes for the threshold with respect to the test instances, often improves the utility compared to the baseline or the EUM style of using a threshold selected with training data. The results for AM and G-TP/PR metrics (using Algorithm \ref{alg:DTAgeneral} for DTA) are presented in Table \ref{fig:table2}. In this case, while choosing a threshold other than 1/2 helps, there is no clear winner between the DTA and the EUM approaches. Overall, our results are consistent with the literature which suggests that threshold optimization results in improved performance. DTA utility optimization outperforms the baselines using some metrics, and results in performance comparable to EUM for others. Additional empirical study is planned for future work.


\begin{table*}
\centering
\begin{tabular}{|p{3cm}| c | r | r | r || r | r | r | }\hline
    {\sc Dataset}& T & DTA & Baseline & EUM & DTA & Baseline & EUM   \\
      & & $F_1$ & $F_1$ & $F_1$ & Jaccard & Jaccard & Jaccard  \\ \hline \hline
& 1 & \textbf{0.5875} & 0.5151 & 0.4980 &\textbf{0.4761} & 0.4308 & 0.4257\\
\textsc{Reuters} & 10 & \textbf{0.8247} & 0.7624 &0.7599 & 0.6801 &0.6409 &\textbf{0.6910} \\
(65) & 50 & \textbf{0.8997} & 0.8428 & 0.8510 &0.7515 & 0.7448& \textbf{0.7578}\\
& 100 & \textbf{0.9856} & 0.9675 & 0.9669 & 0.9398 & 0.9375 &0.9357\\ \hline
\textsc{Letters} (26) & 1 & \textbf{0.7110} & 0.4827 & 0.5745 & 0.4272 & 0.3632 & \textbf{0.4318} \\ \hline
\textsc{Scene} (6) & 1 & \textbf{0.9626} & 0.6891 & 0.5916 & \textbf{0.3540} & 0.0206 & 0.2080\\ \hline
\textsc{Web page} & 1 & \textbf{0.8394} &  0.6269 & 0.6267 & 0.4637 & \textbf{0.5215} & 0.5194\\ \hline
\textsc{Spambase} & 1 & \textbf{0.9636} & 0.8798 & 0.8892 &  0.7314 & 0.7867 & \textbf{0.8003} \\ \hline
\textsc{Image} & 1 & \textbf{0.9578} & 0.8571 & 0.8581 & 0.7455 & 0.7500 & \textbf{0.7623} \\ \hline
\textsc{Breast Cancer} & 1 & \textbf{0.9793} & 0.9589 & \textbf{0.9766} & 0.9342 & 0.9211 & \textbf{0.9481} \\ \hline
\end{tabular}
\caption{Comparison of methods: Linear-fractional metrics, $F_1$ and Jaccard. Baseline refers to thresholding $\hat{\eta}(x)$ at $0.5$; DTA refers to the proposed method of computing $\s^*$ using Algorithm \ref{alg:DTA_FL}; and EUM refers to the plugin-estimator method in \citet{koyejo2014consistent}. First three are multi-class datasets (number of classes indicated in parenthesis): metric is computed individually for each class that has at least $T$ positive instances (in both the train and the test sets) and then averaged over classes.}
\label{fig:table1}
\end{table*}

\begin{table*}
\centering
\begin{tabular}{|p{3cm}| c | r | r | r || r | r | r | }\hline
    {\sc Dataset}& T & DTA & Baseline & EUM & DTA & Baseline & EUM   \\
      & & AM & AM & AM & G-TP/PR & G-TP/PR & G-TP/PR  \\ \hline \hline
& 1 & \textbf{0.8834} & 0.7223  & 0.7733 & \textbf{0.7289} & 0.5447 & 0.5299\\
\textsc{Reuters} & 10 & \textbf{0.9520} & 0.8360 & 0.9111& \textbf{0.8066}& 0.7800& \textbf{0.8076}\\
(65) & 50 &\textbf{0.9659} &0.9017 & 0.9582& 0.8495 & 0.8441 &\textbf{0.8691}\\
& 100 & 0.9783 & 0.9761 & 0.9781 & 0.9687 & 0.9675&0.9672\\ \hline
\textsc{Letters} (26) & 1 &\textbf{0.8715} & 0.7020 & \textbf{0.8720} & 0.5787 & 0.5064 & \textbf{0.5902} \\ \hline
\textsc{Scene} (6) & 1 & \textbf{0.5840} & 0.5065  &  \textbf{0.5810} & \textbf{0.5069} & 0.0605 & 0.3848 \\ \hline
\textsc{Web page} & 1 & 0.8689 &  0.8205 &  \textbf{0.8750}& 0.6617 & \textbf{0.6867} & \textbf{0.6886} \\ \hline
\textsc{Spambase} & 1 & 0.8780 & \textbf{0.9010} & \textbf{0.9090} & 0.8494 & 0.8831 & \textbf{0.8913} \\ \hline
\textsc{Image} & 1 & 0.8041 & \textbf{0.8192} &  0.8069 & 0.8676 & 0.8577 & \textbf{0.8702} \\ \hline
\textsc{Breast Cancer} & 1 & 0.9796 & 0.9661 & \textbf{0.9830} & 0.9660 & 0.9590 & \textbf{0.9734} \\ \hline
\end{tabular}
\caption{Comparison of methods: AM and G-TP/PR metrics. Baseline refers to thresholding $\hat{\eta}(x)$ at $0.5$; DTA refers to the proposed method of computing $\s^*$ using Algorithm \ref{alg:DTAgeneral}; and EUM refers to the plugin-estimator method in \citet{narasimhan2014statistical}. First three are multi-class datasets (number of classes indicated in parenthesis): metric is computed individually for each class that has at least $T$ positive instances (in both the train and the test sets) and then averaged over classes.}
\label{fig:table2}
\end{table*}

\section{Conclusions and Future Work}
\label{conclusions}
The goal of this paper is to bridge a gap in the binary classification literature, between empirical utility maximization (EUM) and decision theoretic analysis. In particular, our analysis shows that many popular metrics satisfy a probability ranking principle, so the DTA optimal classifier is given by the signed thresholding of the conditional probability of the positive class. This result matches a similar analysis in the EUM literature. 

We propose a \regularity/ property for metrics, which if satisfied is sufficient to guarantee that the metric satisfies the probability ranking principle. We show that \regularity/ is satisfied by large families of binary performance metrics including the monotonic family studied by \citet{narasimhan2014statistical}, and a large subset of the linear fractional family studied by \citet{koyejo2014consistent}. We also recover known results for the special cases of $F_\beta$ and SEC. We propose efficient and consistent estimators for optimal expected out-of-sample classification. In particular, we show that as a consequence of the probability ranking principle, computational requirements can be reduced from exponential to cubic complexity in the general case, and further reduced to quadratic complexity in special cases. 

The similarity between the DTA optimal and EUM optimal classifiers suggests a more fundamental connection. Indeed, \citet{ye2012} showed that in the special case of $F_{\beta}$, the DTA and EUM optimal classifiers as asymptotically equivalent as the number of test samples tends to infinity. A similar results can be shown for any classifier that satisfies the probability ranking principle. The details of the result will be included in the extended version of this manuscript. For future work, we plan to extend our analysis to multiclass and multilabel classification, to explore if and when the optimal classifiers take a simple form, and to design efficient classification algorithms.
\clearpage
\bibliography{consistent-classifier}
\bibliographystyle{plainnat}

\clearpage
\appendix
\section{Appendix A}
\subsection{Proof of Theorem~\ref{thm:mainresult1}}
\label{app:mainresult1}
The proof is by contradiction. Fix a distribution $\bbP \in \cP$, and let $\x$ denote a set of $n$ iid samples from the marginal $\bbP_{\X}$. Denote $\bbP(Y = 1|x_{i}) = \eta_{i}$ and the optimal classifier by $\s^{*} \in \Y^n$. Suppose there exist indices $j, k$ such that $s^{*}_{j} = 1, s^{*}_{k} = 0$ and $\eta_j < \eta_k$. Let $\s' \in \Y^n$ be such that $s'_{j} = 0$ and $s'_{k} = 1$, but identical to $\s^*$ otherwise i.e. $s^*_i = s'_i \; \forall i \in [n]\backslash\{j, k\}$. Note that $\sum_{i=1}^{n}s^{*}_{i} =\sum_{i=1}^{n}s^{'}_{i}$. 

By optimality of $\s^{*}$ it is clear that,
\begin{equation}
\label{eq:opt}
\utilDTA(\s^*; \bbP) - \utilDTA(\s'; \bbP) \geq 0.
\end{equation}
Consider the LHS, $\utilDTA(\s^*; \bbP) - \utilDTA(\s'; \bbP)$ is equal to:
\begin{align*}
&\sum_{\y \in \Y^{n}}P(\y|\x) [\MET(\s^{*}, \y) - \MET(\s',\y)]\\
= &\sum_{\y \in \Y^{n}: y_{j} \neq y_{k}}P(\y|\x) [\MET(\s^{*}, \y) - \MET(\s',\y)] \\ 
+ &\sum_{\y \in \{0,1\}^{n}: y_{j} = y_{k}} P(\y|\x) \underbrace{[\MET(\s^{*}, \y) - \MET(\s',\y)]}_{(*)} \\
\end{align*}
Note that when $y_{j} = y_{k} = 0$, $\sum_{i=1}^{n}s^{*}_{i}y_{i} = \sum_{i=1}^{n}s'_{i}y_{i}$, so $\MET(\s^{*}, \y) - \MET(\s',\y) = 0$. It follows that the term $(*)$ equals $0$.

Net we apply the representation of Proposition~\ref{prop:tilpsi} with $v(\s) = \frac{1}{n}\sum_{i=1}^n s_i$ and $p(\y) = \frac{1}{n}\sum_{i=1}^n y_i$.  Let $z \in \{0,1\}^{n-2}$ denote the vector corresponding to $n-2$ indices $\{ y_i, \; i \in [n] \setminus \{ j,k \} \}$, then $\utilDTA(\s^*; \bbP) - \utilDTA(\s'; \bbP)$ is given by:
\begin{align*}
\sum_{\y \in \{0,1\}^{n}: y_{j} \neq y_{k}}\bbP(\y|\x) [\MET(\s^{*}, \y) - \MET(\s',\y)] = \\
\sum_{\z \in \{0,1\}^{n-2}}\bbP(\z, y_{j}=1,y_{k}=0|\x) \big[\ALT(\tph(\s^*,\y), v(\s^*), p(\y)) \\ - \ALT(\tph(\s',\y), v(\s'), p(\y))\big] \\
+ \bbP(\z, y_{j}=0,y_{k}=1|\x) \big[\ALT(\tph(\s^*,\y), v(\s^*), p(\y)) \\ - \ALT(\tph(\s',\y), v(\s'), p(\y))\big]
\end{align*}
Let $\tilde{\s} = \{s^*_i \; \forall i \in [n] \setminus \{ j,k \} \} $ and define $\#TP(\z) := \sum_{i}\tilde{s}_iz_i$ and $\#p(\z) = z_{i}$ (where the $\#$ prefix indicates counts rather than normalized values), and note that $v(\s^*) = v(\s')$. With these substitutions, $\utilDTA(\s^*; \bbP) - \utilDTA(\s'; \bbP)$ is given by:
\begin{align*}
\sum_{\z \in \{0,1\}^{n-2}}\bbP(\z, y_{j}=1,y_{k}=0|\x) \\ \bigg[\ALT\bigg(\frac{1}{n}(\#TP(\z)+1), v(\s'), \frac{1}{n}(\#p(\z) + 1)\bigg) - \\ 
\ALT\bigg(\frac{1}{n}\#TP(\z), v(\s'), \frac{1}{n}(\#p(\z) + 1)\bigg)\bigg] \\
+ \bbP(\z, y_{j}=0,y_{k}=1|\x) \\ \bigg[\ALT\bigg(\frac{1}{n}\#TP(\z), v(\s'), \frac{1}{n}(\#p(\z) + 1)\bigg) - 
\\ \ALT(\frac{1}{n}(\#TP(\z) + 1), v(\s'), \frac{1}{n}(\#p(\z) + 1)\bigg)\bigg]
\end{align*}
Next applying the iid assumption on the labels, we have that $P(\z, y_{j},y_{k}|\x) = P(\z|\x)P(y_{j}|\x)P(y_{k}|\x)$, so that the equation further simplifies to:
\begin{align*}
\sum_{\z \in \{0,1\}^{n-2}}\bbP(\z|\x) \bigg[\ALT\bigg(\frac{1}{n}(\#TP(\z)+1), v(\s'), \\ 
\frac{1}{n}(\#p(\z) + 1)\bigg) - \ALT\bigg(\frac{1}{n}\#TP(\z), v(\s'), \frac{1}{n}(\#p(\z) + 1)\bigg)\bigg] \\ 
\bigg[\eta_j(1-\eta_k) - \eta_k(1-\eta_j)\bigg] = \\
(\eta_j - \eta_k) \sum_{\z \in \{0,1\}^{n-2}}\bbP(\z|\x)  \bigg[\ALT\bigg(\frac{1}{n}(\#TP(\z)+1), v(\s'), \\ 
\frac{1}{n}(\#p(\z) + 1)\bigg) - \ALT\bigg(\frac{1}{n}\#TP(\z), v(\s'), \frac{1}{n}(\#p(\z) + 1)\bigg)\bigg]
\end{align*}
Note that for each $\z \in \{0,1\}^{n-2}$:
\begin{itemize}
\item $\ALT\bigg(\frac{1}{n}(\#TP(\z)+1), v(\s'), \frac{1}{n}(\#p(\z) + 1)\bigg)$ can be interpreted as $\MET$ computed on the vectors $\y \in \mathbb{R}^n$ defined as $\{y_i = z_i \; \forall \; i \in [n]\setminus \{j,k\}\} \cup \{y_j = 1\} \cup \{y_k = 0\}$, and $\s^* \in \mathbb{R}^n$ (which is the assumed optimal).
\item $\ALT\bigg(\frac{1}{n}\#TP(\z), v(\s'), \frac{1}{n}(\#p(\z) + 1)\bigg)$ can be interpreted as $\MET$ computed on the vectors $\y \in \mathbb{R}^n$ defined as above and $\s' \in \mathbb{R}^n$.
\end{itemize}
By \regularity/ of $\MET$, for each $\z$, the difference term  $\ALT\bigg(\frac{1}{n}(\#TP(\z)+1), v(\s'), \frac{1}{n}(\#p(\z) + 1)\bigg) - \ALT\bigg(\frac{1}{n}\#TP(\z), v(\s'), \frac{1}{n}(\#p(\z) + 1)\bigg) > 0$. This combined with \eqref{eq:opt} implies that $\eta_j - \eta_k \geq 0$ which is a contradiction.

\subsection{Proof of Proposition~\ref{prop:connection}}
\label{app:connection}
Suppose $\MET$ satisfies \tpregularity/. Let $u_1 = \tp(\s_1, \y_1)$ and $u_2 = \tp(\s_2, \y_2)$, $v = v(\s_1) = v(\s_2)$ and $p = p(\y_1) = p(\y_2)$. Note that $\ALT(u_1,v,p) = \REG(\frac{u_1}{p}, \frac{1-v-p+u_1}{1-p}, p)$ (and similarly equality holds for $\ALT(u_2,v,p)$). Now, whenever $u_1 = \tph(\s_1,\y_1) > \tph(\s_2,\y_2) = u_2$, $v(\s_1) = v(\s_2) = v$, and $p(\y_1) = p(\y_2) = p$, we have $\tprh(\s_1,\y_1) > \tprh(\s_2,\y_2), \tnrh(\s_1,\y_1) > \tnrh(\s_2,\y_2)$, and
\begin{align*}
\ALT(u_1,v,p) &=& \REG(\frac{u_1}{p}, \frac{1 - v - p + u_1}{1 - p}, p) \\
&=& \REG(\tprh(\s_1,\y_1), \tnrh(\s_1,\y_1), p) \\
&\overset{(*)}{>}& \REG(\tprh(\s_2,\y_2), \tnrh(\s_2,\y_2), p) \\
&=& \REG(u_2 . p, \frac{1 - v - p + u_2}{1 - p}, p) \\
&=& \ALT(u_2, v, p)
\end{align*}
where $(*)$ follows from \tpregularity/ of $\MET$. Thus $\MET$ satisfies \regularity/.

\section{Appendix B}
\subsection{Proof of Theorem~\ref{thm:consistent}}
\label{app:consistent}
Let $\,\utilDTA_* := \utilDTA(\s^*; \bbP)$ and let $\,\utilDTAhat = \utilDTA(\hat{\s}; \bbP)$. Also define the empirical distribution: 
\[\hat{\bbP}(\y|\x) = \Pi_{i=1}^n \hat{\eta}(x_i)^{y_i}(1 - \hat{\eta}(x_i))^{1 - y_i}.\]

Now consider:
\begin{align}
\label{eqn:interstep1}
\utilDTA_* - \utilDTAhat &=& \utilDTA_* - \utilDTA(\hat{\s}; \hat{\bbP}) + \utilDTA(\hat{\s}; \hat{\bbP}) - \utilDTAhat \nonumber \\
&\leq& \utilDTA_* - \utilDTA(\s^*; \hat{\bbP}) + \utilDTA(\hat{\s}; \hat{\bbP}) - \utilDTAhat  \nonumber\\
&\leq& 2 \max_{\s} \big| \utilDTA(\s; \bbP) - \utilDTA(\s; \hat{\bbP})\big|
\end{align}

For any fixed $\s \in \{0,1\}^n$, we have:
\begin{align}
\label{eqn:interstep2}
&\big| \utilDTA(\s; \bbP) - \utilDTA(\s; \hat{\bbP})\big| \nonumber \\
&= \big|\sum_{y \in \{0,1\}^n} \hat{\bbP}(\y|\x) \MET(\s, \y) - \sum_{y \in \{0,1\}^n} \bbP(\y|\x) \MET(\s, \y) \big| \nonumber \\
&\leq \sum_{y \in \{0,1\}^n} \big| \hat{\bbP}(\y|\x) - \bbP(\y|\x)\big| \MET(\s, \y)
\end{align}
Let $\eta(x)$ denote the empirical estimate obtained using $m$ training samples. Now because $\hat{\eta}(x) \convinp \eta(x)$, we have that for sufficiently large set of training examples, $\hat{\bbP}(\y|\x) \convinp \bbP(\y|\x)$; i.e. for any given $\epsilon > 0$, there exists $m_\epsilon$ such that for all $m > m_\epsilon$, $|\hat{\bbP}(\y|\x) - \bbP(\y|\x)| < \epsilon$, with high probability. It follows that, with high probability, $\eqref{eqn:interstep2} \leq \epsilon \sum_{y \in \{0,1\}^n} \MET(\s,\y)$. Assuming $\MET$ is bounded, we have that for any fixed $s$, $\big| \utilDTA(\s; \bbP) - \utilDTA(\s; \hat{\bbP})\big| \leq C \epsilon$, for some constant $C$ that depends only on the metric $\MET$ and (fixed) test set size $n$. The uniform convergence also follows because the $\max$ in \eqref{eqn:interstep1} is over finitely many vectors $\s$. Putting together, we have that for any given $\delta, \epsilon' > 0$, there exists training sample size $m_{\epsilon',\delta}$ such that the output $\hat{\s}$ of our procedure satisfies, with probability at least $1 - \delta$, $\utilDTA_* - \utilDTAhat < \epsilon'$. The proof is complete.

\end{document}